\documentclass{article}
\usepackage{fullpage}
\usepackage{natbib}
\usepackage{microtype}
\usepackage{graphicx}
\usepackage{subfigure}
\usepackage{booktabs} 
\usepackage{hyperref}

\usepackage{amsmath}
\usepackage{amsfonts}
\usepackage{amsthm}
\usepackage{mathtools}
\usepackage{thmtools}
\usepackage{soul}
\usepackage{algorithm}
\usepackage[noend]{algpseudocode}

\usepackage{enumitem}
\setlist{nolistsep}

\newcommand \argmax {\operatorname{argmax}}
\newcommand \Lap {\operatorname{Lap}}
\newcommand \PrivateSplit {\operatorname{PrivateSplit}}
\newcommand \depth {\operatorname{depth}}

\newcommand \leaves {\operatorname{leaves}}
\newcommand \Range {\operatorname{Range}}

\newtheorem{theorem}{Theorem}
\newtheorem{corollary}[theorem]{Corollary}
\newtheorem{lemma}[theorem]{Lemma}
\theoremstyle{definition}
\newtheorem{definition}[theorem]{Definition}
\declaretheoremstyle[%
  spaceabove=-6pt,%
  spacebelow=6pt,%
  headfont=\normalfont\itshape,%
  postheadspace=1em,%
  qed=\qedsymbol%
]{mystyle}
\declaretheorem[name={Proof Sketch},style=mystyle,unnumbered,
]{proofsketch}

\begin{document}
\title{
Scalable and Provably Accurate Algorithms for Differentially Private Distributed Decision Tree Learning
}
\author{Kaiwen Wang\thanks{Correspondence to: \url{wangkaiwen998@gmail.com}}\\
       Computer Science Department\\
       Carnegie Mellon University\\
       Pittsburgh, PA 15213, USA
       \and
       Travis Dick\thanks{Most of this work was done while at the School of Computer Science, Carnegie Mellon University}\\
       Department of Computer and Information Science\\
       University of Pennsylvania\\
       Philadelphia, PA 19104, USA
       \and
       Maria-Florina Balcan\\
       School of Computer Science\\
       Carnegie Mellon University\\
       Pittsburgh, PA 15213, USA}

\date{}
\maketitle

\begin{abstract}
This paper introduces the first provably accurate algorithms for differentially private, top-down decision tree learning in the distributed setting \citep{balcan2012distributed}. We propose DP-TopDown, a general privacy-preserving decision tree learning algorithm, and present two distributed implementations. Our first method NoisyCounts naturally extends the single machine algorithm by using the Laplace mechanism. Our second method LocalRNM significantly reduces communication and added noise by performing local optimization at each data holder.
We provide the first utility guarantees for differentially private top-down decision tree learning in both the single machine and distributed settings.
These guarantees show that the error of the privately-learned decision tree quickly goes to zero provided that the dataset is sufficiently large. Our extensive experiments\footnote{Code available at \url{https://github.com/kaiwenw/DPDDT}} on real datasets illustrate the trade-offs between privacy, accuracy and generalization when learning private decision trees in the distributed setting.
\end{abstract}

\section{Introduction}
New technologies for edge computing and the Internet of Things (IoT) are driving computing toward the distributed setting \citep{satyanarayanan2017emergence}. A distributed learning algorithm has access to data with more quantity and diversity than learning from one data holder.
Today, data entities comprise of mobile phones, schools, hospitals, companies, countries, and more.
As an illustrative example, consider training a model for early detection of a new strand of malaria, which has been tested in multiple hospitals across the country. Most hospitals alone would not have enough data to train an accurate model. Even if a hospital could locally train a model, the model would be biased toward and only accurate for that specific city. As such, it would perform poorly as a nationwide diagnostic tool. However, while benefits of distributed learning are evident, privacy concerns can often deter cooperation.

Differential privacy \citep{Dwork06:Calibrating} is a strong guarantee of individual privacy with an additional benefit of encouraging generalization \citep{dwork2015generalization, dwork2015reusable, dwork2015preserving}. Differentially private mechanisms formally quantify the amount of leaked information and can greatly incentivize distributed learning with sensitive data. Another attractive feature of differential privacy is its robustness to post-processing; once the model is learned through a differentially private mechanism, any arbitrary applications of the model as well as post-processing to the model can be performed without worrying about additional privacy leaks.

In this work, we investigate learning decision trees while preserving privacy in the distributed setting \citep{balcan2012distributed}. Decision trees are often used for their efficiency and interpretability. The learned structure of decision trees is much more interpretable than other machine learning models like deep neural networks. Motivated by the fact that constructing optimal decision trees is NP-complete \citep{hyafil76:hardTrees}, greedy top-down learning algorithms are used in practice (i.e. ID3, C4.5 and CART \cite{mitchell1997machine}). Top-down algorithms can be viewed as boosting algorithms that quickly drive their error to zero under the Weak Learning Assumption \citep{kearns1999boosting}.

In the distributed setting, any communicated information between data holders must be made differentially private, which limits the quantity and quality of information that can be published to learn the decision tree. Intermediate computations often cost privacy, which violates the information minimization principle of differential privacy \citep{dwork2014algorithmic}. Despite these constraints, we are able to prove theoretical guarantees that the error of the privately-learned tree quickly goes to zero with high probability provided we have sufficient data under the Weak Learning Assumption.

We present DP-TopDown, a family of differentially private decision tree learning algorithms based off the TopDown algorithm of \citet{kearns1999boosting}. DP-TopDown uses a differentially private subroutine, called $\PrivateSplit$, to approximate the optimal split chosen by TopDown. Designing a private and distributed mechanism that satisfies the utility requirements of $\PrivateSplit$ is the main challenge of running DP-TopDown.
In our first approach, called \textit{NoisyCounts}, each entity publishes noised counts that are aggregated to find the best splitting functions. Our second approach \textit{LocalRNM} is a heuristic that performs local optimization at each entity to greatly reduce the amount of communication (and thus noise) compared to NoisyCounts.

Finally, we comprehensively evaluate our different implementations of DP-TopDown on real, relevant datasets, and provide plots that illustrate the privacy-accuracy trade-off for our algorithms. Empirically, LocalRNM performs better than NoisyCounts on most datasets. Motivated by our theory, we also show both the training and testing accuracy of the decision tree increases with more training data. In one dataset, we observe that private algorithms actually improve the test accuracy, which is similar in spirit to adding some noise in non-convex optimization problems to escape local optimums \citep{zhang2017hitting}.

Our contributions are summarized as follows:
\begin{enumerate}
    \item We propose DP-TopDown, a differentially private, top-down learning algorithm for decision trees, and present two implementations for the distributed setting.
    \item We prove boosting-based utility guarantees for DP-TopDown. These are the first utility guarantees for differentially private top-down learning of decision trees in both single machine and distributed settings.
    \item We extensively evaluate DP-TopDown on real datasets, \href{https://github.com/kaiwenw/DPDDT}{reproducible with open-sourced code}\footnote{Code available at \url{https://github.com/kaiwenw/DPDDT}}.
\end{enumerate}

\subsection{Related Works}
Previous works have explored private top-down learning of decision trees for specific algorithms. \cite{blum2005practical} introduced and used the SuLQ framework to implement ID3. \cite{friedman2010data,mohammed2011differentially} demonstrated the effectiveness of using the exponential mechanism for private ID3 and C4.5. Our approach generalizes previous works and provide the first utility guarantees and sample complexity bounds for private, top-down decision tree learning, even in the distributed setting.

Apart from top-down decision tree learning (i.e. ID3 and CART), previous works have also considered random decision trees, where the construction of the decision tree is entirely random and can be done even before it sees any data \citep{jagannathan2009practical}. \cite{bojarski2014differentially} studied techniques which averaged over collections of random decision trees (a.k.a. Random Forests) and gave guarantees for empirical and generalization errors. In practice, random forests are powerful models that sometimes out-perform vanilla top-down learned decision trees. However, top-down learning of decision trees performs well in many cases and in general leads to more interpretable trees than randomly constructed decision trees.

Previous works have used boosting to design improved differentially private algorithms for answering batches of queries \citep{dwork2010boosting}. Our work is different and is about designing differentially private boosting algorithms, in particular top-down decision tree learning.

\section{Preliminaries}
We study learning problems where the goal is to map input examples from a space $X$ to a finite output space $Y$. We suppose that there is an unknown distribution $P$ defined over $X$, together with an unknown target function $f : X \to Y$. Given a sample $S$ of i.i.d. pairs $(x, f(x))$ drawn from the distribution $P$, our goal is to learn an approximation to $f$.

We focus on the canonical problem of top-down decision tree learning. In a decision tree, each internal node is labeled by a splitting function $h : X \to \{0,1\}$ that ``splits'' the data according to $h(x)$, and the leaves of the tree are labeled by a classes in $Y$. These splitting functions route each example $x \in X$ to exactly one leaf of the tree. That is, if the root of the tree has splitting function $h$, a point $x$ is routed to the left subtree if $h(x) = 0$ and to the right subtree if $h(x) = 1$. The prediction made by a decision tree $T$ for an input $x$ is the label of the leaf that $x$ is routed to.
We let $H \subset \{0,1\}^X$ denote the class of splitting functions that may be placed at each internal node, and denote the error of a decision tree $T$ by $\varepsilon(T) = \Pr(T(x) \neq f(x))$, where $x$ is sampled from $P$.

\subsection{Top-down Decision Tree Learning}
In this section we describe the high level structure of (non-private) top-down decision tree learning algorithms. Top-down learning algorithms construct decision trees by repeatedly ``splitting'' a leaf node in the tree built so far. On each iteration, the algorithm greedily finds the leaf and splitting function that maximally reduces an upper bound on the error of the tree. The upper bound (see Equation~\eqref{eq:upperBound}) on error is induced by a \emph{splitting criterion} $G : [0,1] \to [0,1]$ that is concave, symmetric about $\frac{1}{2}$, and that satisfies $G(\frac{1}{2}) = 1$, $G(0) = G(1) = 0$. The selected leaf is replaced by an internal node labeled with the chosen splitting function, which partitions the data at the node into two new children leaves. Once the tree is built, the leaves of the tree are labeled by the label of the most common class that reaches the leaf. Many effective and widely-used decision tree learning programs are instances of the TopDown algorithm. For example, C4.5 uses entropy for splitting criterion while CART uses Gini.

The splitting criterion $G$ provides an upper bound on the error of a decision tree $T$. The error $\varepsilon(T)$ can be written as
$ 
    \varepsilon(T) = \sum_{\ell \in \leaves(T)} w(\ell) \min\{q(\ell), 1-q(\ell)\}
$ 
where $w(\ell) = \Pr_P[x \text{ reaches leaf } \ell]$ is the fraction of data that reaches $\ell$ and $q(\ell) = \Pr_P[f(x) = 1 | x \text{ reaches leaf } \ell]$ is the fraction of that data with label 1. For all $z \in [0,1]$, the splitting criterion satisfies $G(z) \geq \min\{z, 1-z\}$, and therefore the following potential function is an upper bound on the error $\varepsilon(T)$ of a decision tree $T$:
\begin{align}
    G(T) &= \sum_{\ell \in \leaves(T)} w(\ell) G(q(\ell)).
    \label{eq:upperBound}
\end{align}
With this, we can formally describe one iteration of top-down decision tree learning. For each leaf $\ell \in \leaves(T)$, let $S_{\ell} \subseteq S$ denote the subset of data that reaches leaf $\ell$. For each splitting function $h \in H$, let $T(\ell,h)$ denote the tree obtained from $T$ by replacing leaf $\ell$ by an internal node that splits $S_{\ell}$ into two children leaves $\ell_0, \ell_1$, where any data satisfying $h(x) = i$ goes to $\ell_i$. At each round, TopDown chooses the split-leaf pair $(h^*, \ell^*)$ that maximizes the decrease in the potential:
$
    G(T) - G(T(\ell,h)) = w(\ell)J(\ell, h)
$
where
$
J(\ell, h) = G(q(\ell)) - \frac{|S_{\ell_0}|}{|S_\ell|}G(q(\ell_0)) - \frac{|S_{\ell_1}|}{|S_\ell|} G(q(\ell_1))
$.

Kearns and Mansour showed that (non-private) top-down decision tree learning can be viewed as a form of boosting \citep{kearns1999boosting}. Whenever the class of splitting functions $H$ used to label internal nodes satisfies the Weak Learning Assumption (i.e., they can predict the target $f$ better than random guessing for any distribution on $X$), then top-down learning algorithms amplify these slight advantages and quickly drives the training error of the tree to zero. Formally, the Weak Learning Assumption is as follows:
\begin{definition}
\label{def:wla}
A function class $H \subset \{0,1\}^X$ $\gamma$-satisfies the \emph{Weak Learning Assumption} for a target $f : X \to \{0,1\}$ if, for any distribution $P$ over $X$, $\exists h \in H, \Pr_{x \sim P}(h(x) \neq f(x)) \leq \frac{1}{2} - \gamma$. $\gamma$ is called the advantage.
\end{definition}

\noindent We restate the result of \cite{kearns1999boosting} when $G$ is entropy: $G(q) = -q\log(q)-(1-q)\log(1-q)$.
\begin{theorem}
\label{thm:kmthm}
Let $G(q)$ be entropy. Under the Weak Learning Assumption, for any $\varepsilon \in (0,1]$, TopDown suffices to make $(1/\varepsilon)^{O(\log(1/\varepsilon)/\gamma^2)}$ splits to achieve training error $\leq \varepsilon$.
\end{theorem}

\subsection{Differential Privacy}
Differential privacy is the \textit{de facto} notion of privacy \citep{Dwork06:Calibrating}. The formal notion is as follows. We say two datasets $S$ and $S'$ are neighboring if they differ on at most one point. A randomized algorithm $\mathcal{A}$ is $\alpha$-differentially private if for any neighboring datasets $S, S'$ and any set of outcomes $\mathcal{O} \subseteq \Range(\mathcal{A})$,
\begin{align*}
    \Pr(\mathcal{A}(S) \in \mathcal{O}) \leq e^{\alpha} \Pr(\mathcal{A}(S') \in \mathcal{O})
\end{align*}
where the probability space is over the randomness of $\mathcal{A}$.

We recall two standard differentially private mechanisms.
First, the Laplace mechanism (LM) can be used to evaluate a function $f$ mapping datasets to real numbers, while preserving $\alpha$-differential privacy. Given a dataset $S$, the Laplace mechanism with privacy parameter $\alpha$ outputs $f(S) + Z$, where $Z$ is drawn from the Laplace distribution with scale parameter $\Delta_f / \alpha$ and $\Delta_f$ is the \emph{sensitivity} of $f$:
\begin{align*}
    \Delta_f = \max_{S, S' \text{ neighboring datasets}} |f(S) - f(S')|.
\end{align*}
We will also use Report Noisy Max (RNM), an alternative to the exponential mechanism. Given a collection of scoring functions $f_1, \dots, f_k$, each mapping datasets to real-valued scores, RNM privately estimates $i^* = \argmax_i f_i(S)$ and the score $f_{i^*}(S)$. Given a dataset $S$, RNM with privacy parameter $\alpha$ first computes $\hat f_i = f_i(S) + Z_i$, where $Z_i$ is drawn from the Laplace distribution with scale parameter $\max_{i = 1, \dots, k} 2\Delta_{f_i} / \alpha$. Then, RNM outputs $i^* = \argmax_{i} \hat f_i$ together with the estimate $\hat f_{i^*}$. Importantly, RNM does not publish the estimated scores $\hat f_i$ for $i \neq i^*$, allowing it to add less noise than applying $k$ LMs.

Our decision tree learning algorithms will apply LM and RNM many times, with varying privacy parameters. Standard composition theorems, presented in Appendix \ref{sec:DPComposition}, guarantee that the overall algorithms are still $\alpha$-differentially private \citep{Dwork06:Calibrating,mcsherry2009privacy}.

\subsection{Distributed Setting with Privacy Constraints}
In this work, we consider the distributed learning setting with privacy constraints \citep{balcan2012distributed}. We suppose there are $k$ entities (i.e. data holders) and each entity $i \in [k]$ must preserve $\alpha$-differential privacy for their data $S^i$. In general, we make no assumptions on how $S^i$ partitions $S = \bigcup_i S^i$ in terms of size and distribution. An example of this setting is hospitals with private patient data. Patients entrust their sensitive information to be used within their local hospital, but the response to any queries from outside the hospital must preserve privacy on behalf of the patients. Note that this notion is not local differential privacy \citep{warner1965randomized,cormode2018privacy}, where patient data would be perturbed before made available to the hospital.

We remark that distributed learning algorithms can also use cryptographic privacy protocols such as multi-party computation to preserve privacy while learning \citep{evans2018pragmatic}. Assuming reasonable computational limitations of the adversary, cryptographic privacy leaks zero bits of information during the training process. On the other hand, differential privacy is an information theoretic guarantee that bounds the amount of leaked information and is robust to post-processing. In this work, we focus on differential privacy protocols since they tend to be more communication efficient and simpler to implement in practice. 

\section{Differentially Private TopDown}
In this section, we propose a private, top-down decision tree learning algorithm called DP-TopDown that is based off the TopDown algorithm \citep{kearns1999boosting}. The key step of DP-TopDown that depends on the sensitive dataset $S$ is choosing the optimal split for a given leaf $\ell$. Since there are many ways to do this, we encapsulate this with a function called $\PrivateSplit$ that satisfies the following privacy and utility requirements, which are sufficient to prove privacy and utility guarantees for DP-TopDown.

\noindent\textbf{Requirements for $\PrivateSplit$: }
\vspace{0.2em}
\begin{enumerate}
  \item $\PrivateSplit(\ell, \alpha, \delta)$ outputs a pair $(\hat h, \hat J) \in H \times \mathbb{R}$, while preserving $\alpha$-differential privacy;
    \item $\forall \zeta > 0, \exists N := N(\zeta, \alpha, \delta)$ so that if $|S_{\ell}| \geq N$, then $J(\ell,\hat{h}) \geq \max_h J(\ell,h) - \zeta$ and $|J(\ell,\hat{h}) - \hat{J}| \leq \zeta$ with probability $\geq 1-\delta$. In words, this means that $\PrivateSplit$ finds nearly optimal splits w.h.p.
\end{enumerate}
\vspace{0.2em}
Pseudocode for DP-TopDown is given in Algorithm \ref{alg:privateDT}. It takes as input the dataset $S$, the privacy parameter $\alpha$, the max number of nodes (excluding leaves) in the tree $M$, the desired error $\varepsilon \in (0,1]$ and failure probability $\delta \in (0,1]$. We also introduce a function $\mathcal{B}(d)$ that allocates the privacy budget for nodes at depth $d$. The function $\mathcal{B}$ must satisfy $\sum_{d=1}^M \mathcal{B}(d) \leq 1$, which will ensure that the total privacy cost of building the layers of our decision tree does not exceed our privacy budget. These parameters are used in Theorem \ref{thm:dptopdown}.

\begin{algorithm}
\caption{DP-TopDown}
\label{alg:privateDT}
\begin{algorithmic}[1]
  \State \textbf{Input:} dataset $S$, max nodes $M$, error $\varepsilon$, privacy $\alpha$, failure probability $\delta$
    \State Init $T$ to single leaf $\ell$ and $Q$ to max-priority queue
    \State $(\hat{h}, \hat{J}) \gets \PrivateSplit(\ell, \frac{\alpha}{2}\mathcal{B}(1), \frac{\delta}{2(2M+1)})$
    \State Push $(\ell, \hat{h})$ to $Q$ with priority $\hat{J}$
    \For{$t = 1,2,...,M$}
        \State $(\ell^*, h^*) \leftarrow$ Pop element of $Q$ with highest priority
        \State $T \gets T(\ell^*, h^*)$ with children leaves $\ell_0, \ell_1$
        \For{$i = 0,1$}
            \State $\alpha_{\ell} \gets \frac{\alpha}{2}\mathcal{B}(\depth(\ell_i))$ [budget for leaf]
            \State $w_i \leftarrow \frac{|S_{\ell_i}|}{|S|} +
        \Lap(\frac{2}{|S|\alpha_{\ell}})$ [LM with budget $\frac{\alpha_{\ell}}{2}$] \label{line:weightestimate}
            \State $(\hat h_i, \hat J_i) \gets \PrivateSplit(\ell_i, \frac{\alpha_{\ell}}{2}, \frac{\delta}{2(2M+1)})$
            \If{$w_i \geq \frac{\varepsilon}{M}$} \label{line:weightlowerbound}
                \State Push $(\ell_i, \hat h_i)$ to $Q$ with priority $w_i \cdot \hat J_i$
            \EndIf
        \EndFor
    \EndFor
    \State Label leaves by majority label [RNM with budget $\frac{\alpha}{2}$] \label{line:labelleaf}
    \State \Return $T$
\end{algorithmic}
\end{algorithm}

\textbf{Privacy guarantee:}
The privacy analysis follows from the composition theorems. We use $\frac{\alpha}{2}$ for labeling the leaves at line \ref{line:labelleaf} using LM. Since the leaves partition $S$, this preserves $\frac{\alpha}{2}$-differential privacy by parallel composition (Theorem \ref{thm:parallelcomp}). As nodes at depth $d$ also partition $S$, splits at depth $d$ preserve $\mathcal{B}(d)$-differential privacy. Then, by sequential composition (Theorem \ref{thm:seqcomp}), the internal splits together preserve $\frac{\alpha}{2}$-differential privacy if $\sum_{\text{all depths } d} \mathcal{B}(d) \leq 1$. Since the depth of our trees is upper bounded by $M$, the maximum number of non-leaf nodes, it is sufficient for $\sum_{d=1}^M \mathcal{B}(d)$ to be bounded by 1. 

One privacy budgeting approach is to uniformly split the budget across all possible depths with $\mathcal{B}(d) = \frac{1}{M}$. This gives us the tightest sample complexity bounds for Theorem \ref{thm:dptopdown}.
Another approach is to set $\mathcal{B}(d) = \frac{1}{2^{d}}$.
This exponentially decay budgeting is motivated by the intuition that early splits in the tree are more important than later splits. We found that the decay budgeting strategy tends to give higher accuracy than uniform budgeting (see Figure \ref{fig:alpha_trainingfraction_test_acc_paper_fig}).
There is a trade-off between budgeting for leaf labeling and splitting internal nodes, which we explore empirically in Figure \ref{fig:alpha_trainingfraction_test_acc_paper_fig}. Using half the total budget for leaf labeling can be excessive for large $\alpha$, and a better tree could be learned if more budget was allocated to splitting nodes. 

\vspace{0.5em}\textbf{Utility guarantee:}
Intuitively for any fixed $\alpha$, the noise from preserving $\alpha$-differential privacy becomes negligible with high probability as the size of the dataset grows large. The following theorem formalizes this intuition for DP-TopDown by providing general utility guarantees and sample complexity bounds, which are applicable in both the single machine and distributed settings.

\begin{restatable}{theorem}{thmBoosting}
\label{thm:dptopdown}
Let $G(q)$ be entropy. Let $M, \alpha, \varepsilon, \delta$ be inputs to Algorithm \ref{alg:privateDT}. Under the Weak Learning Assumption with advantage $\gamma$, DP-TopDown suffices to make $(1/\varepsilon)^{O(\log(1/\varepsilon)/\gamma^2)}$ splits to achieve training error $\leq \varepsilon$ with probability $\geq 1-\delta$ provided the dataset $S$ has size at least
\[
|S|
\geq \max\left(
  \widetilde O\left(\frac{M}{\gamma^2 \varepsilon \alpha b }\right),
  \frac{2M}{\varepsilon}N\left(\widetilde O(\frac{\gamma^2 \varepsilon}{M}),
  \frac{\alpha b}{4},
  O(\frac{\delta}{M})\right)
\right)
\]
where $b = \min_{1 \leq d \leq M} \mathcal{B}(d)$ and $N(\zeta, \alpha, \delta)$ is the data requirement of $\PrivateSplit$.
\end{restatable}
\begin{proofsketch}
Let $T_t$ denote the decision tree constructed by DP-TopDown so far at the beginning of the $t^\text{th}$ iteration. The boosting analysis of \cite{kearns1999boosting} has two key steps. First, they argue that during the $t^\text{th}$ iteration, there must be a leaf $\ell$ of the tree $T_t$ whose weight is at least $\varepsilon(T_t)/t$. Next, they leverage the weak-learning assumption to guarantee that the optimal split for that leaf significantly reduces the potential $G_t$. This guarantees that after a small number of iterations, the error cannot be large.

There are two main challenges in extending this analysis to the DP-TopDown algorithm. First, since DP-TopDown only splits leaves whose \emph{estimated} weight is large, there is some risk that it will not consider splitting the leaf used in the analysis of Kearns and Mansour. Second, even if that leaf is considered, the $\PrivateSplit$ algorithm will only produce nearly optimal splits for that leaf when there is enough data reaching that leaf (recall that $\PrivateSplit$ returns an $\zeta$-optimal split with probability at least $1-\delta$, provided that there are at least $N(\zeta, \alpha, \delta)$ data points in the leaf).

The term $\widetilde O(\frac{M}{\gamma^2 \varepsilon \alpha b})$ in our dataset size requirements is chosen to ensure that the error in the estimated weight of every leaf during the run of DP-TopDown is bounded by $\frac{\gamma^2 \varepsilon}{M}$ w.h.p. This guarantees that the leaf used in the analysis of Kearns and Mansour will be considered by DP-TopDown w.h.p. Moreover, every leaf applying $\PrivateSplit$ has non-trivial weight and the error from labeling leaves is less than $\frac{\varepsilon}{2}$.

A consequence of the above argument is that DP-TopDown only runs the $\PrivateSplit$ algorithm on leaves with weight at least $\Omega(\frac{\varepsilon}{M})$. In other words, they contain at least $\Omega(\frac{|S| \varepsilon}{M})$ data points. The term $\frac{2M}{\varepsilon}N\bigl(\widetilde O(\frac{\gamma^2 \varepsilon}{M}), \frac{b\alpha}{2}, O(\frac{\delta}{M})\bigr)$ in our dataset size requirement is chosen to guarantee the number of data points in any leaf that we run $\PrivateSplit$ is large enough to guarantee that it returns a $\zeta$-optimal split where $\zeta$ is at most half the reduction in $G(T)$ guaranteed by the split constructed in the Kearns and Mansour analysis. This ensures that our algorithm reduces the potential $G(T)$ by at least half of the guarantee of the non-private algorithm, which is sufficient to prove our utility guarantee.
\end{proofsketch}

We note that our proof generalizes to other splitting criteria analyzed by \cite{kearns1999boosting}, such as the Gini Index $G(q) = 4q(1-q)$ and $G(q) = \sqrt{2q(1-q)}$ (see Appendix \ref{sec:otherCriterion}).
We also note that the only assumption on $H$ in Theorem \ref{thm:dptopdown} is the Weak Learning Assumption. There is no assumption on the size of $H$. In this paper, we present implementations of $\PrivateSplit$ that iterate through $H$, which is why our corollaries depend on $|H|$ being finite.

\begin{figure*}
\begin{center}
\centerline{\includegraphics[width=0.87\textwidth]{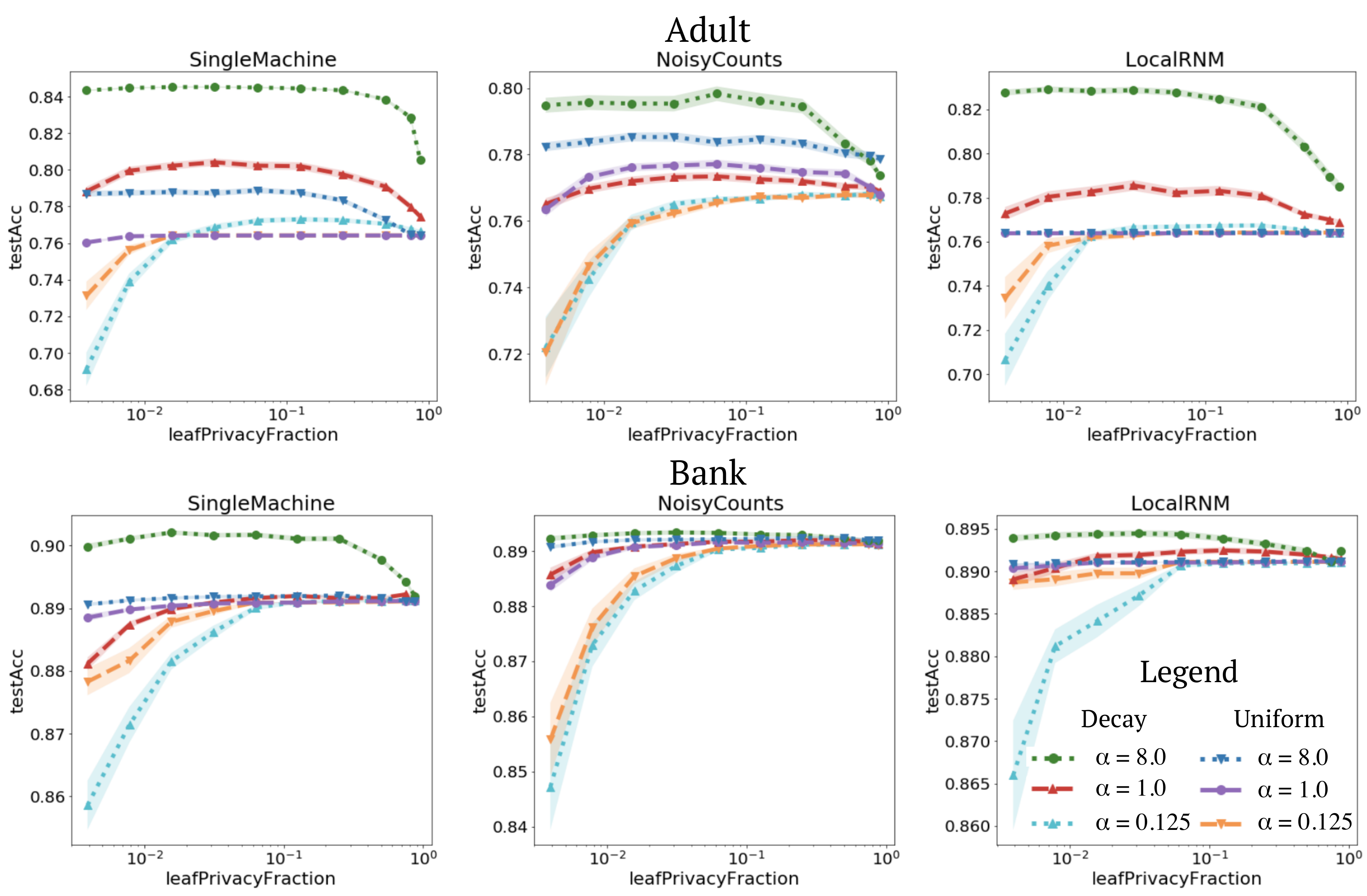}}
\vspace*{-5mm}
\caption{Test accuracy against LPF for decay and uniform budgeting strategies, and varying values of $\alpha \in \{0.125, 1, 8\}$.
}
\label{fig:alpha_trainingfraction_test_acc_paper_fig}
\end{center}
\vskip -0.35in
\end{figure*}

\vspace{0.5em}\noindent\textbf{SingleMachine algorithm: RNM.}
In the single machine setting and when the class of splitting functions $H$ is finite, we can simply use RNM to select an approximately optimal split from $H$. In particular, if we define $f_i(S_\ell) = J(\ell,h_i)$ for $h_i \in H$ and use RNM to choose the optimal split, we satisfy the utility requirement of $\PrivateSplit$ as follows.
\begin{lemma}\label{lm:rnmprivatesplit}
Single machine RNM satisfies the $\PrivateSplit$ requirements with $N(\zeta, \alpha, \delta) = \widetilde{O}(\frac{1}{\alpha\zeta})$.
\end{lemma}
\begin{proofsketch}
Lemma~\ref{lm:entropysensitivity} of the Appendix shows that the sensitivity of $f_i(S_\ell) = J(\ell, h_i)$ is bounded by $\widetilde O(1/|S_\ell|)$. Let $\hat f_i = f_i(S) + Z_i$ where $Z_i$ is drawn from the Laplace distribution with scale parameter $O(\alpha / |S_\ell|)$ to each estimate $\hat f_i$. With probability at least $1-\delta$, we have $|Z_i| = O(\frac{\alpha}{|S_\ell|} \log \frac{|H|}{\delta} )$ simultaneously for all $i$. When $|S_\ell| \geq \widetilde \Omega(\frac{1}{\alpha \zeta})$, the error in all estimates made by the RNM mechanism are bounded by $\zeta/2$.
\end{proofsketch}
\begin{corollary}
Under uniform budgeting, using RNM as $\PrivateSplit$ requires $|S| \geq \widetilde{O}(\frac{M^3}{\varepsilon^2 \gamma^2 \alpha})$ for Theorem \ref{thm:dptopdown}.
\end{corollary}

\noindent We remark that another private implementation is with the exponential mechanism, which recovers the algorithm in \cite{friedman2010data,mohammed2011differentially}.

\vspace{0.5em}\noindent\textbf{Distributed setting with privacy constraint.}
Weight estimation in line \ref{line:weightestimate} can be done distributedly by having each entity estimate the number of data points at a particular leaf with LM. Leaf labeling can be done distributedly by first estimating the number of data points at a particular leaf for each label using LM and then finding the label that maximizes the aggregated counts. The sample complexity of Theorem \ref{thm:dptopdown} accrues a $k$ factor in the first term, as formally quantified in Corollary \ref{thm:distributeddptopdown} of the Appendix. Therefore, we can run DP-TopDown in the distributed setting given a distributed $\PrivateSplit$. We now present two methods.

\vspace{0.5em}\noindent\textbf{Distributed algorithm 1: NoisyCounts.}
For each splitting function $h_i \in H$, each entity $j \in [k]$ uses LM to publish noisy estimated counts $\hat c_{y=a, h_i=b}^{j}$ of the number of datapoints at leaf $\ell$, with label $a$ and splitting value $b$. Then, the coordinator can use the aggregated noisy counts to compute each $\hat J(\ell, h_i)$ and locate the maximum. The following lemma follows from the utility guarantees of LM.
\begin{lemma}\label{lm:noisycountspaper}
NoisyCounts satisfies the $\PrivateSplit$ requirements with $N(\zeta, \alpha, \delta) = \widetilde{O}(\frac{k|H|}{\alpha\zeta})$.
\end{lemma}
\begin{proofsketch}
The proof has two main steps. First, we bound the total error in the noisy counts collected by the coordinator, and second we argue that when the counts are sufficiently accurate, then the chosen splitting function is nearly optimal. The collection of counts associated with any splitting function $h_i \in H$ form a histogram query, and they can all be approximated while satisfying $\alpha'$-differential privacy using the Laplace mechanism by adding $\Lap(\frac{1}{\alpha'})$ noise to each count. We approximate one histogram query for each $h_i \in H$, so we set $\alpha' = \frac{1}{|H|}$ so that the total privacy cost is $\alpha$. Next, using tail bounds for Laplace random variables, we know that the maximum error in any of the counts constructed by a single entity is bounded by $\widetilde O(\frac{|H|}{\alpha})$ w.h.p. Therefore, each accumulated count obtained by the coordinator has at most $\widetilde O(\frac{k |H|}{\alpha})$ error accumulated from the $k$ entities. Finally, when $|S_\ell| \geq \widetilde O(\frac{k|H|}{\alpha \zeta})$, the counts are accurate enough guarantee a $\zeta$-optimal split.
\end{proofsketch}
\begin{corollary}
Under uniform budgeting, using NoisyCounts as $\PrivateSplit$ requires $|S| \geq \widetilde{O}(\frac{M^3k|H|}{\varepsilon^2 \gamma^2 \alpha})$ for Theorem \ref{thm:dptopdown}.
\end{corollary}

\noindent Compared to single machine, NoisyCounts' bound on $N$ increased by a $k|H|$ factor. This is since each entity publishes all the counts so that $\hat J(\ell, h_i)$ can be approximated for every $h_i \in H$ whereas only the maximum is needed. Therefore, each of the $k$ entities calls LM with $|H|$ factor more noise. Unlike weight estimation and leaf labeling, which could extend to the distributed setting while accruing only a $k$ factor, NoisyCounts accrues an additional $|H|$ factor since there are $|H|$ splitting functions to try. This is problematic since $|H|$ may be much larger than $k$ in general.

\vspace{0.5em}\noindent\textbf{Distributed algorithm 2: LocalRNM.}
To reduce communication and to increase the privacy budget for each data query, we propose the following heuristic improvement. With budget $\frac{\alpha}{2}$, each entity $i$ uses RNM to find the locally-best splitting function $\hat{h}^*_i$ on their own data $S_{\ell}^{i}$. Then, with budget $\frac{\alpha}{2}$, the coordinator runs NoisyCounts to find the best splitting function over the set of locally-best splitting functions $\hat H = \{\hat{h}^*_i\}_{i \in [k]}$ rather than the entire $H$.

\vspace{0.5em}Here, the noise of each data query scales with the number of entities $k$ instead of the number of splitting functions $|H|$. This is in general much less than the noise added by NoisyCounts. However, locally-best splitting functions, especially from entities with small datasets, can have poor performance across the union of data, which is why LocalRNM does not satisfy the utility requirements of $\PrivateSplit$ in general. Nevertheless, LocalRNM performs noticeably better than NoisyCounts on four out of seven datasets. This suggests that at least one of the locally-best splitting functions tend to be good enough in practice.

\section{Experimental Evaluation}
\label{sec:results}

\begin{figure*}
\begin{center}
\centerline{\includegraphics[width=1\textwidth]{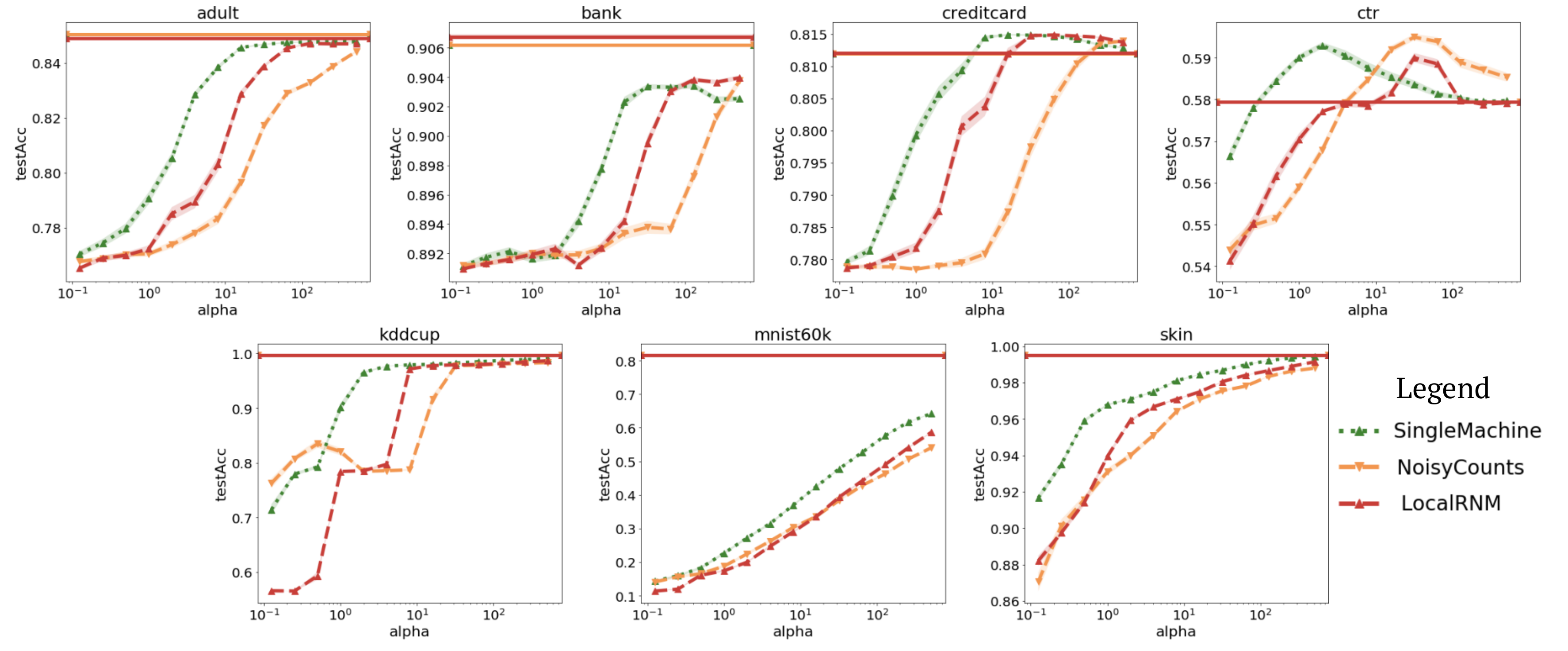}}
\vspace*{-5mm}
\caption{Test accuracy against $\alpha$ for SingleMachine, NoisyCounts and LocalRNM. Baseline runs without noise are shown as solid lines.
}
\label{fig:algos_test_acc_all}
\end{center}
\vskip -0.35in
\end{figure*}

\begin{figure*}[ht]
\begin{center}
\centerline{\includegraphics[width=0.93\textwidth]{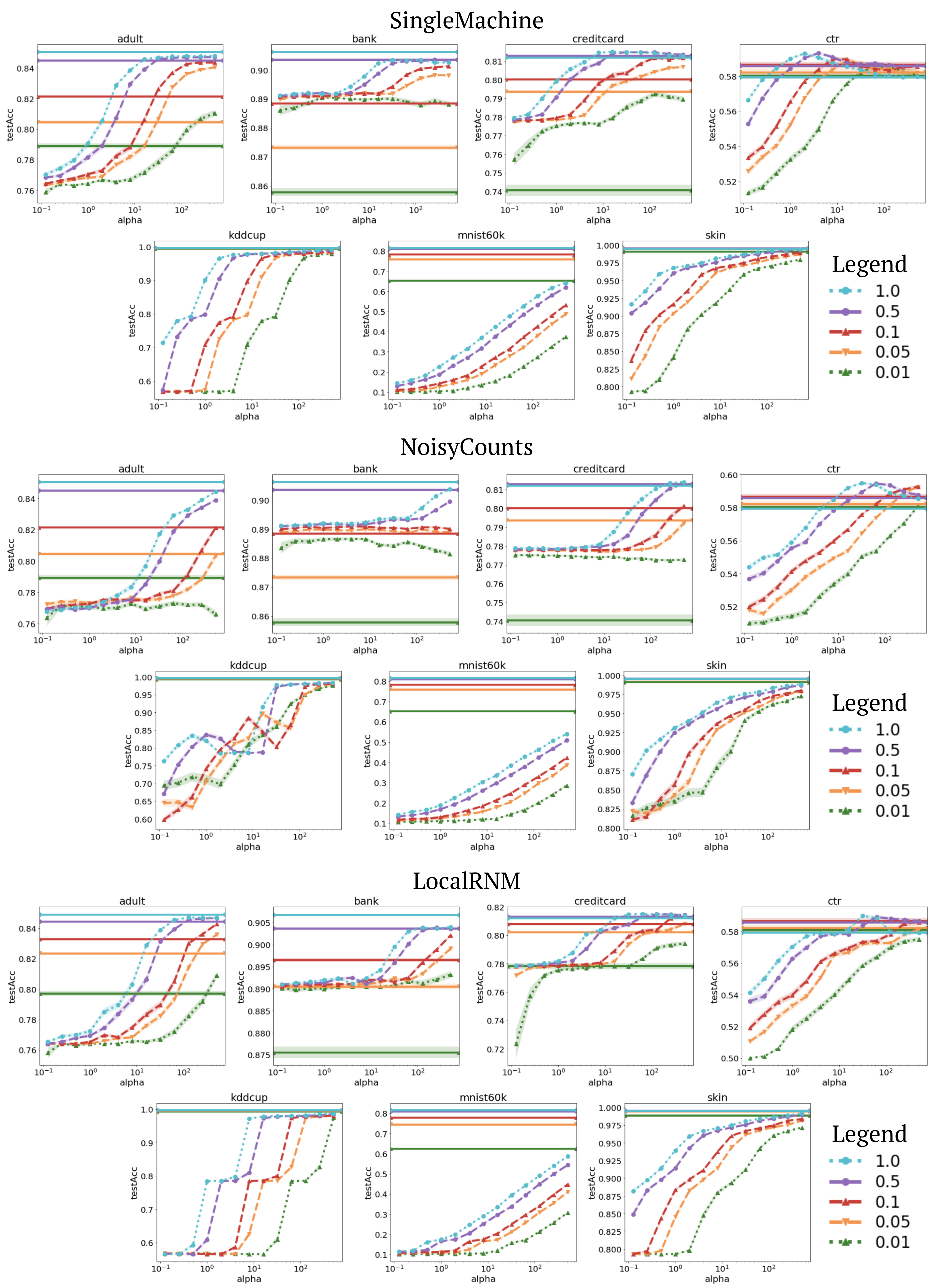}}
\vspace*{-5mm}
\caption{
Test accuracy against $\alpha$ for differently sized training sets. For example, a 0.01 curve were trained on a random $1\%$ subset of the total dataset. Baseline runs without noise are solid lines in the same color as corresponding dotted lines, representing noised runs. }
\label{fig:datasize_test_acc_all}
\end{center}
\vskip -0.3in
\end{figure*}

We evaluated and compared our three DP-TopDown algorithms on seven datasets exhibiting many applications where privacy could be of concern, especially in the distributed setting. We one-hot encoded categorical features. The splitting class for each dataset comprised of evenly spaced thresholds for each feature, with just one threshold at $0.5$ for the one-hot encoded features. We now give a brief description of each dataset.

\vspace{0.5em}\noindent
\textbf{Adult}: US Census Data to predict if one's income exceeds $\$50K$ \citep{Dua:2019}. Total size: $32,561$. \\
\textbf{Bank}: Predict subscription behavior from marketing campaign info \citep{moro2014data}. Total size: $45,211$. \\
\textbf{KDDCup 1999}: Classify network connections as malicious or normal \citep{cup1999dataset}. Total size: $494,021$. \\
\textbf{MNIST}: Classify handwritten digit from $28 \times 28$ greyscale image \citep{lecun1998mnist}. Total size: $70,000$. \\
\textbf{Creditcard}: Client payment information to predict default outcome \citep{yeh2009comparisons}. Total size: $30,000$. \\
\textbf{Avazu CTR}: From anonymized mobile ad info, predict if ad was clicked \citep{CRTDataset}. Total size: $1,100,000$. \\
\textbf{Skin}: Classify RGB pixels of image as skin or non-skin \citep{bhatt2010skin}. Total size: 245,057.
\vspace{1.5em}

MNIST was divided into fixed train/test sets with size ratio $6:1$. The other datasets were $9:1$. Adult, Bank, Creditcard, KDDCup used 10 evenly spaced thresholds for continuous features, giving $|H| = 159, 115, 210, 427$ respectively. Skin used 32 evenly spaced thresholds, giving $|H| = 96$. MNIST, CTR had $|H| = 147, 263$ respectively, described in Appendix \ref{sec:splittingclassdesc}.
We used entropy splitting criterion, $M=512$, $\varepsilon=0.1$ and $\alpha \in \{2^{-3}, 2^{-2}, ..., 2^9\}$. Unless stated otherwise, decay budgeting was used. Data was divided among four entities uniformly randomly. When $\hat J \leq 0.01$, we did not push $\ell$ into $Q$, as it suggested $\ell$ was mostly homogeneous or $\hat J$ was too noisy. Our results were averaged over 100 independent runs on AWS Batch. Shaded error bars denote one standard deviation in the mean, which is tiny in most plots.

Figure \ref{fig:alpha_trainingfraction_test_acc_paper_fig} shows the budgeting trade-off between leaf labeling and internal splits on Adult and Bank. The x-axis is LeafPrivacyFraction (LPF), the fraction of $\alpha$ allocated for leaf labeling. As LPF approaches 1, leaf labeling becomes more accurate but the internal splits become more random. However, both need to be sufficiently accurate to learn an accurate tree, hence why performance suffers as LPF approaches 0 or 1.
As LPF approaches 0, accuracy degrades more with the decay budget than the uniform budget, which suggests the decay budget is more sensitive to noisy leaf labels.
In practice, searching for an optimal LPF consumes privacy, so we fixed $LPF=0.5$ in other experiments. Appendix \ref{sec:more_results_for_alpha_trainingfraction_fig} presents results on other datasets.

Figure \ref{fig:algos_test_acc_all} shows the trade-off between privacy and test accuracy (i.e. privacy curves) for all three private DP-TopDown algorithms presented. The (dotted) accuracy of private algorithms increases as a function of privacy parameter $\alpha$ and eventually converges to the (solid) non-private baseline with no noise. Overall, single machine performs the best. In the distributed setting, LocalRNM, while lacking utility guarantees, performs consistently better than NoisyCounts on Adult, Bank, Creditcard and Skin. Both trends illustrate our intuition that injecting noise should worsen performance.

Counter-intuitively, the private algorithms yielded higher test accuracy than their non-private counterparts on CTR for some $\alpha$. We found that non-private decision trees had higher error due to some unlucky greedy choices, which were avoided by adding small amounts of noise through private mechanisms. Our result shows that injecting noise into TopDown, prone to overfitting, can give better accuracy. This is similar to adding noise in non-convex optimization problems to escape local optima \citep{zhang2017hitting}.
When designing private greedy algorithms, there is trade-off not only between privacy and accuracy, but also generalization.

Figure \ref{fig:datasize_test_acc_all} shows the effects of training dataset size on the test accuracy privacy curves. The curve shapes are regular and larger training sizes consistently result in better test accuracy, which confirms our intuition from Theorem \ref{thm:dptopdown}, that DP-TopDown enjoys the same utility guarantees as TopDown when trained on large datasets. We also observe that smaller dataset sizes generally lead to higher variance in the error bars, which is expected since the injected noise has greater relative affect to the data queries. For small datasets, the benefit of injecting noise for generalization is now very clear. For example, NoisyCount's Bank plot shows that (dotted) private curves of 0.01 and 0.05 consistently outperform the (solid) non-private baselines.

Appendix \ref{sec:more_results_for_algos_test_acc_all}, \ref{sec:more_results_for_datasize_test_acc_all} presents more interesting results. Figure \ref{fig:datasize_train_acc_all} shows that training on larger datasets also leads to higher training accuracy, which is counter-intuitive since smaller datasets are easier to overfit. 

\section{Conclusion}
In this paper, we unify the theory and practice of differentially private top-down decision tree learning in the distributed setting. We provide the first utility guarantees for such algorithms as well as a comprehensive experimental analysis on seven relevant datasets. Our design of DP-TopDown reduces the problem of private top-down learning to the key challenge of approximating optimal splits with $\PrivateSplit$. Using Theorem \ref{thm:dptopdown}, it is easy to derive utility guarantees for other implementations of $\PrivateSplit$. With the code that will be made available, this is a great starting point for applying scalable and provably accurate differentially private decision trees in the distributed setting.

\noindent\textbf{Acknowledgements}:
This work was supported in part by NSF grants CCF-1535967, IIS-1618714,  CCF-1910321, SES-1919453, and an Amazon Research Award.

\clearpage
\bibliography{main}

\appendix
\section{Appendix}

\subsection{Composition Theorems of Differential Privacy}
\label{sec:DPComposition}
\begin{theorem}[Sequential Composition] \label{thm:seqcomp}
Let $A_1, \dots, A_k$ be $k$ mechanisms such that $A_i$ satisfies $\alpha_i$-differential privacy. Then the mechanism $A(S) = (A_1(S), \dots, A_k(S))$ satisfies $\left(\sum_i \alpha_i\right)$-differential privacy. This holds even if $A_i$ is chosen based on the outputs of $A_1(S), \dots, A_{i-1}(S)$.
\end{theorem}

\begin{theorem}[Parallel Composition]
\label{thm:parallelcomp}
  Let $D_1, \dots, D_k$ be a partition of the input domain and suppose $A_1, \dots, A_k$ are mechanisms such that $A_i$ satisfies $\alpha_i$-differential privacy. Then the mechanism $A(S) = (A_1(S \cap D_1), \dots, A_k(S \cap D_k))$ satisfies $\left(\max_i \alpha_i\right)$-differential privacy.
\end{theorem}

\subsection{Sensitivity of Entropy}
In this section, we analyze the sensitivity of entropy. We will use the following standard result.
\begin{lemma}
For any $x, y \in (0,1)$, if $|x-y| \leq \frac{1}{2}$, then $|x\lg(x)-y\lg(y)| \leq -|x-y|\lg(|x-y|)$.
\end{lemma}

We now deduce the global sensitivity of $J$, used in our implementation of PrivateSplit with RNM.
In the proof below, we use the notation $J(S,h)$ where $S$ is a dataset, which is more general than $J(\ell, h)$ (equivalently $J(S_\ell, h)$) for leaf $\ell$.
\begin{lemma}
\label{lm:entropysensitivity}
Suppose $G(q) = -q\lg(q) - (1-q)\lg(1-q)$, the entropy function. When applied to a dataset of size $m$, then for any $h \in H$, the sensitivity of $J'(S) = J(S, h)$ is $\Delta J' \leq \frac{2}{m}(3\lg(m) + 1)$.
\end{lemma}
\begin{proof}
For a dataset $S$ of size $m$ and splitting function $h$, denote $r_{y=i,h=j} = \frac{|\{x \in S| f(x)=i, h(x)=j\}|}{m}, r_{y=i} = \sum_{j \in \{0,1\}} r_{y=i,h=j}$ and similarly for $r_{h=j}$. Since the numerator is a count that changes by at most 1 for neighboring datasets $S$ and $S'$, we have that $|r_*-r'_*| \leq \frac{1}{m}$ (here, $*$ denotes wildcard). Consider an arbitrary $h \in H$ and fix $h$. Then,
\begin{align*}
    &|J(S,h)-J(S',h)| \\
    &= |-[\sum_{a \in \{0,1\}} r_{y=a}\lg(r_{y=a}) - r'_{y=a}\lg(r'_{y=a})] \\
    &+ [ r_{h=0} \sum_{a \in \{0,1\}} r_{y=a,h=0} \lg(r_{y=a,h=0}) \\
    &- r'_{h=0} \sum_{a \in \{0,1\}} r'_{y=a,h=0} \lg(r'_{y=l,h=0}) ] \\
    &+ [ r_{h=1} \sum_{a \in \{0,1\}} r_{y=a,h=1} \lg(r_{y=l,h=1}) \\
    &- r'_{h=1} \sum_{a \in \{0,1\}} r'_{y=a,h=1} \lg(r'_{y=a,h=1}) ] | \\
    &\leq \frac{2}{m} \lg(m) + 2(\frac{1}{m} + \frac{2}{m}\lg(m))
\end{align*}
by triangle inequality, the above lemma, and $-x\log(x)$ is increasing when $x \in (0,e^{-1})$ (assuming $\frac{1}{m} \leq e^{-1}$).
\end{proof}

\subsection{RNM Implements SingleMachine PrivateSplit}
\label{sec:rnmprivatesplit}
We first recall the utility guarantee for Laplace mechanism \citep{dwork2014algorithmic}.
\begin{lemma}
\label{lm:laplaceutility}
If $Y \sim \Lap(b)$ then $\Pr(|Y| \geq \ln(\frac{1}{\delta}) \cdot b) = \delta$. In general, if $Y_1, Y_2, ..., Y_k$ are i.i.d. samples from $\Lap(b)$, then $\Pr(\max_{i} |Y_i| \geq \ln(\frac{k}{\delta}) \cdot b) \leq \delta$.
\end{lemma}

We first restate the RNM algorithm for $(\hat h, \hat J) = \PrivateSplit(\ell, \alpha, \delta)$. For every $i = 1,2,...,|H|$, define $f_i(S_{\ell}) = J(\ell, h_i)$. By Lemma \ref{lm:entropysensitivity}, $\Delta f_i \leq 10\frac{\lg(m)}{m}$. Thus, RNM samples $X_i \sim \Lap(\frac{20\lg(m)}{\alpha m})$, calculates $\hat f_i = f_i(S_{\ell}) + X_i$ and computes $i^* = \arg \max_{i \in [|H|]} \hat{f}_i$, while preserving $\alpha$-differential privacy. Now we show that this procedure finds nearly optimal splits w.h.p.

\begin{lemma}
\label{lm:technicallemma}
Let $b > 0$. If $x \geq 2b\log(b)$, then $x \geq b\log(x)$ \citep{shalev2014understanding}.
\end{lemma}

\begin{theorem}
\label{thm:rnmprivatesplit}
The single machine RNM-based procedure described above satisfies: $\forall \zeta > 0, \exists N = \widetilde{O}(\frac{1}{\alpha \zeta})$ so that if $|S_{\ell}| \geq N$, then with probability at least $1-\delta$, we have $J(\ell,\hat{h}) \geq \max_h J(\ell,h) - \zeta$ and $|J(\ell,\hat{h}) - \hat{J}| \leq \zeta$.
\end{theorem}

\begin{proof}
We show that a stronger claim holds with high probability. By Lemma \ref{lm:laplaceutility}, $|X_i| = |f_i - \hat f_i| \leq \frac{\zeta}{2}, \forall i \in [|H|]$ holds with probability $1-\delta$ whenever $\ln(\frac{|H|}{\delta})\frac{20\lg(m)}{\alpha m} \leq \frac{\zeta}{2}$ for large enough $m$. By Lemma \ref{lm:technicallemma}, it suffices to have $m \geq 2b\log(b)$ where $b = \ln(\frac{|H|}{\delta}) \cdot \frac{40}{\alpha \zeta}$.

We now show that $|f_i - \hat f_i| \leq \frac{\zeta}{2}, \forall i$ is sufficient. This is equivalent to $|J(\ell, h_i) - \hat J(\ell, h_i)| \leq \frac{\zeta}{2}, \forall i$. In particular, $|J(l, \hat h) - \hat J| \leq \frac{\zeta}{2} \leq \zeta$.
Let $j^* = \arg \max_{i \in [s]} J(\ell, h_i)$. Then, $f_{i^*} \geq f_{j^*} - \zeta$, which is equivalent to $J(\ell, \hat h) \geq \max_h J(\ell, h) - \zeta$. This is because $f_{j^*} - f_{i^*} \leq (f_{j^*}-\hat f_{j^*}) + (\hat f_{j^*} - f_{i^*}) \leq \frac{\zeta}{2} + (\hat f_{i^*}- f_{i^*}) \leq \zeta$.
Therefore $N(\zeta, \alpha, \delta) = \widetilde{O}(\frac{1}{\alpha \zeta})$.
\end{proof}

\subsection{NoisyCounts Implements Distributed PrivateSplit}
\label{sec:dblprivatesplit}
In NoisyCounts, each entity $j$ estimates counts $\hat c_{y=a,h_i=b}^j, \hat c_{y=a}, \hat c_{h_i=b}$ for every splitting function $h_i$ using the LM with $\frac{3\alpha}{|H|}$ privacy budget for each type of count. Since counts have sensitivity of $1$, the noise added for LM is sampled from $\Lap(\frac{|H|}{3\alpha})$. By Theorem \ref{thm:seqcomp}, this entire operation preserves $\alpha$-differential privacy.
The coordinator aggregates the counts by summing over the entities $\hat c_{y=a, h_i=b} = \sum_{j \in [k]} \hat c_{y=a,h_i=b}^j$. Then, the aggregated noisy counts are used to compute each $\hat J(\ell,h_i)$ and locate the maximum.

\begin{theorem}
The NoisyCounts procedure described above satisfies: $\forall \zeta > 0$, $\exists N = \widetilde{O}(\frac{k|H|}{\alpha \zeta})$ so that if $|S_\ell| \geq N$ then with probability at least $1-\delta$, we have $J(\ell, \hat h) \geq \max_h J(\ell, h) - \zeta$ and $|J(\ell, \hat h)-\hat J| \leq \zeta$.
\end{theorem}
\begin{proof}
The idea is that the estimated counts are accurate w.h.p. Then, applying analysis similar to the sensitivity of $J$, we know that inputting these accurately estimated counts produces accurate estimates of $J$.
Any estimated count $\hat c_{y=a, h_i=b}$ is the true count $c_{y=a,h_i=b}$ noised by $X_1 + X_2 + ... + X_k$ where each $X_i \sim \Lap(\frac{3|H|}{\alpha})$. Since w.p. $1-\frac{\delta}{3k|H|}$, we have $|X_i| \leq \ln(\frac{3k|H|}{\delta})\frac{3|H|}{\alpha}$, applying union bound yields that $|\hat c_{y=a,h_i=b} - c_{y=a,h_i=b}| \leq \ln(\frac{3k|H|}{\delta})\frac{3k|H|}{\alpha}$ for every splitting function $h_i$, w.p. $\geq 1-\frac{\delta}{3}$. This is also true for $c_{y=a}$ and $c_{h_i=b}$. Using notation from the proof of Lemma \ref{lm:entropysensitivity}, $|r_* - r_*'| \leq \ln(\frac{3k|H|}{\delta})\frac{3k|H|}{m\alpha}$. The same analysis leads to $|\hat J(\ell, h_i) - J(\ell, h_i)| \leq 10 \ln(\frac{3k|H|}{\delta})\frac{3k|H|}{m\alpha} \ln(\frac{m\alpha}{\ln(\frac{3k|H|}{\delta}) 3k|H|}) \leq 30 \ln(\frac{3k|H|}{\delta})\lg(m)\frac{k|H|}{m\alpha}$, assuming $\ln(\frac{3k|H|}{\delta})\frac{3k|H|}{m\alpha} \in (0,\frac{1}{e})$. Thus, w.p. $\geq 1-\delta$: $|\hat J(\ell, h_i) - J(\ell, h_i)| \leq \frac{\zeta}{2}, \forall i$ when $m \geq 2b\log(b)$ with $b = 60\ln(\frac{3k|H|}{\delta}) \frac{k|H|}{\alpha \zeta}$. This is sufficient, as shown in the proof of Theorem \ref{thm:rnmprivatesplit}.
Therefore, $N(\zeta, \alpha, \delta) = \widetilde{O}(\frac{k|H|}{\alpha \zeta})$.
\end{proof}

\subsection{Proof of Main Theorem \ref{thm:dptopdown}}
\label{sec:dptopdownproof}
Theorem \ref{thm:dptopdown} is our main boosting-based utility guarantee. We first define some notation and recount key arguments used to prove boosting-based utility guarantees for TopDown \citep{kearns1999boosting}.

Define $\varepsilon_t = \varepsilon(T_t)$ and $G_t = G(T_t)$ where $T_t$ is the decision tree at the $t^{th}$ iteration for DP-TopDown. Let $t \in [M]$ be arbitrary but fixed step in the algorithm. After $t$ splits, there exists leaf $\ell$ such that $w(\ell)\min(q(\ell),1-q(\ell)) \geq \frac{\varepsilon_t}{t}$, since the tree $T_t$ has $t$ leaves. By using the weak learning assumption, \cite{kearns1999boosting} shows that splitting the leaf $\ell$ with the optimal splitting function (i.e., the split $h \in H$ that maximizes $J(\ell, h)$) reduces $G_t$ significantly: $G_{t+1} \leq G_t - \frac{\gamma^2 G_t}{4t\log(2/G_t)}$. A solution to this recurrence is given by $G_t \leq e^{-\gamma \sqrt{\log(t)/c}}$ for constant $c$, so it suffices to make $(1/\varepsilon)^{O(\log(1/\varepsilon)/\gamma^2)}$ splits.

Now we proceed with the proof of our main Theorem \ref{thm:dptopdown}. Then, we deduce the distributed version in \ref{thm:distributeddptopdown}.
\thmBoosting*
\begin{proof}
Let $\zeta \in (0, \frac{\varepsilon}{2M})$ be a target accuracy for the guarantees of $\PrivateSplit$. We will set a value for $\zeta$ later in the proof. Recall that $\alpha_{\ell} = \frac{\alpha}{2}\mathcal{B}(\depth(\ell))$ is the privacy budget for $\ell$.
We now case on each event:
\vspace{0.5em}
\begin{itemize}
    \item There are at most $2M+1$ calls to $\PrivateSplit$. By the properties of $\PrivateSplit$, each call has low error when $|S_{\ell}| \geq N(\zeta, \frac{\alpha_{\ell}}{2}, \frac{\delta}{2(2M+1)})$ with probability at least $1-\frac{\delta}{2(2M+1)}$; that is $J(\ell, \hat h) \geq \max_h J(\ell, h) - \zeta$ and $|J(\ell, \hat h) - \hat J| \leq \zeta$. Thus, all the calls have low error w.p. at least $1 - \frac{\delta}{2}$.
    \item There are also $2M$ weight estimates on line \ref{line:weightestimate}, and we want all of them to have error $\zeta$ w.p. at least $1-\frac{\delta}{4}$. So we want the weight estimate has low error $|\hat w(\ell) - w(\ell)| \leq \zeta$ w.p. $1-\frac{\delta}{8M}$. By Lemma \ref{lm:laplaceutility}, we need $|S| \geq \ln(\frac{8M}{\delta}) \frac{2}{\zeta \alpha_\ell}$.
    \item There are also at most $M+1$ leaves. We want to bound the total error of leaf labeling by $\frac{\varepsilon}{2}$ w.p. at least $1-\frac{\delta}{4}$, so bound each leaf's error by $\frac{\varepsilon}{4(M+1)}$ w.p. at least $1 - \frac{\delta}{4(M+1)}$. Thus, even if we mistakenly picked a label other than the most common label, the label we did pick gives error at most $\frac{\varepsilon}{2(M+1)}$ away from the most common label. RNM with budget $\frac{\alpha}{2}$ requires $|S| \geq \ln(\frac{4(M+1)}{\delta})\frac{8(M+1)}{\varepsilon \alpha}$.
\end{itemize}
\vspace{0.5em}
By union bound, all of the above happens with probability at least $1-\delta$, so assume this high probability event for the rest of the proof. Note that since $\zeta < \frac{\varepsilon}{2M}$ and $\alpha_\ell < \alpha$, the sample complexity of leaf labeling is asymptotically bounded by the sample complexity of weight estimates.

Now we show that our algorithm chooses leaves with enough data to satisfy the sample complexity of $\PrivateSplit$. We know after $t$ splits, there exists a leaf $\ell_t$ s.t. $w(\ell_t)\min(q(\ell_t),1-q(\ell_t)) \geq \frac{\varepsilon_t}{t}$. Since $\min(q(\ell_t),1-q(\ell_t)) \leq \frac{1}{2}$, we have $w(\ell_t) > \frac{2\varepsilon_t}{t} \geq \frac{2\varepsilon}{M}$. Restricting $\zeta < \frac{\varepsilon}{2M}$, then $\hat w(\ell_t) \geq w(\ell_t) - \zeta \geq \frac{\varepsilon}{M}$; so $\ell_t$ is pushed onto $Q$, which implies that DP-TopDown will consider splitting this leaf. Furthermore, any leaf $\ell$ in $Q$ satisfies $w(\ell) \geq \frac{\varepsilon}{2M}$ since $\zeta < \frac{\varepsilon}{2M}$. In other words, there is at least one leaf in the queue $Q$ with weight large relative to $\varepsilon$, and no leaves that have very small weight. Therefore, to satisfy the sufficient conditions for $\PrivateSplit$ to achieve low error when run on a leaf $\ell$, we require that $|S_\ell| \geq N(\zeta, \frac{\alpha_\ell}{2}, \frac{\delta}{2(2M+1)})$, which is guaranteed to be satisfied for any leaf in the queue $Q$ when $|S| \geq \frac{2M}{\varepsilon}N(\zeta, \frac{\alpha_\ell}{2}, \frac{\delta}{2(2M+1)})$, since $|S_\ell| \geq \frac{\varepsilon}{2M} |S|$.

Next, we argue that every iteration of DP-TopDown reduces the potential $G(T)$ by at least half of the non-private algorithm. This is due to a combination of the weak learning assumption, together with the fact that for every leaf $\ell \in Q$ we are guaranteed that $\PrivateSplit$ returned a $\zeta$-optimal split for that leaf. Let $J^* = \max_h J(\ell, h)$ and $(\hat h, \hat J)$ be the output of $\PrivateSplit$ run on leaf $\ell$, then
\begin{align*}
    &|w(\ell) J^* - \hat w(\ell) \hat J| \\
    &\leq w(\ell) |J^* - \hat J| + \hat J |w(\ell) - \hat w(\ell)| \\
    &\leq |J^* - J(\ell, \hat h)| + |J(\ell, \hat h) - \hat J| + |w(\ell) - \hat w(\ell)| \\
    &\leq 3\zeta
\end{align*}
where we used that $w(\ell),\hat J$ are bounded by 1. Note that the above is true for any estimated $\hat J$ from $\PrivateSplit$; in particular it's true for the leafs in $Q$. Therefore, after $t$ splits, the leaf $l^*$ we pop from $Q$ has $\hat w \hat J$ nearly as good as the optimal split for leaf $\ell$. Hence, applying the Kearns and Mansour weak-learning analysis, we are guaranteed that $G_{t+1} \leq G_t - \frac{\gamma^2 G_t}{4t\log(2/G_t)} + 6\zeta$. Choose $\zeta$ such that $6\zeta = \frac{\gamma^2\varepsilon}{8M\log(2/\varepsilon)}$. Then we have $G_{t+1} \leq G_t - \frac{\gamma^2 G_t}{8t \log(2/G_t)}$. Then, the solution to the recurrence inequality differs only by a constant. Thus, to achieve error at most $\frac{\varepsilon}{2}$, it suffices to make $(2/\varepsilon)^{O(\log(2/\varepsilon)/\gamma^2)}$ splits, which simplifies to $(1/\varepsilon)^{O(\log(1/\varepsilon)/\gamma^2)}$. As we bounded the leaf labeling error by $\frac{\varepsilon}{2}$, the total error of DP-TopDown is at most $\varepsilon$ as desired.

For the above claim to hold, the size of the dataset needs to be at least
\begin{align*}
|S| &\geq \max_{\ell} \{\ln(\frac{8M}{\delta}) \frac{2}{\zeta \alpha_\ell}, \frac{2M}{\varepsilon}N(\zeta, \frac{\alpha_\ell}{2}, \frac{\delta}{2(2M+1)})\}
\end{align*}
where the maximum is taken over all leaves $\ell$ that are encountered by DP-TopDown (i.e., the leaves of any tree $T_t$). The terms in the above maximum only depend on the leaves $\ell$ through the privacy budget $\alpha_\ell$ assigned to each leaf. Since the learned decision tree has depth at most $M$, we are guaranteed that $\alpha_\ell = \frac{\alpha}{2}\mathcal{B}(\depth{\ell}) \geq \frac{\alpha b}{2}$, where $b = \min_{1 \leq d \leq M} \mathcal{B}(d)$. Substituting this into the above bound on $|S|$ completes the proof.
\end{proof}

We note that the same style of analysis also applies for the splitting criterion $G(q) = 4q(1-q)$ or $G(q) = 2\sqrt{q(1-q)}$ as in \cite{kearns1999boosting}.

\begin{corollary}\label{thm:distributeddptopdown}
Suppose the distributed setting. Given a distributed implementation of $\PrivateSplit$,
DP-TopDown has the guarantees of Theorem \ref{thm:dptopdown} with sample complexity
\[
|S|
\geq \max\left(
  \widetilde O\left(\frac{kM}{\gamma^2 \varepsilon \alpha b }\right),
  \frac{2M}{\varepsilon}N\left(\widetilde O(\frac{\gamma^2 \varepsilon}{M}),
  \frac{\alpha b}{4},
  O(\frac{\delta}{M})\right)
\right).
\]
\end{corollary}
\begin{proof}
We are only concerned with formally quantifying the effects of weight estimation and leaf labeling.
\vspace{0.5em}
\begin{itemize}
    \item Weight estimation can be done distributedly by having each entity publish an estimate of the number of data points at a particular leaf with LM. The budget remains the same as single machine from parallel composition of the entities.
    \item Leaf labeling can be done distributedly by having each entity publish an estimate of the number of data points for each label, with LM. Since we are no longer using RNM, the noise parameter increases by factor of 2 (since there are two possible labels).
\end{itemize}
\vspace{0.5em}
In both cases, instead of bounding the magnitude of one Laplace distributed random variable, we bound the sum of $k$ independent Laplace distributed random variables to be less than some error bound $\zeta$ with failure probability $\delta$. We can do so by bounding each error $\frac{\zeta}{k}$ with failure probability $\frac{\delta}{k}$, which Lemma \ref{lm:laplaceutility} provides the sample complexity for. We see in both cases, we accrue a $k$ factor inside and outside the $\ln$ term. In particular, weight estimation requires $|S| \geq \ln(\frac{8kM}{\delta})\frac{2k}{\zeta \alpha_\ell}$ and leaf labeling requires $|S| \geq \ln(\frac{4k(M+1)}{\delta})\frac{8k(M+1)}{\varepsilon \alpha}$.
Thus, to account for the distributed setting in Theorem \ref{thm:dptopdown}, the first term in max accrues an extra $k$ factor, becoming $\widetilde{O}(\frac{kM}{\gamma^2 \varepsilon \alpha b})$.
\end{proof}

\subsection{Splitting Classes for MNIST and Avazu CTR}
\label{sec:splittingclassdesc}
In this section we describe the splitting class we used for MNIST and Avazu CTR.

\vspace{0.5em}
\noindent \textbf{MNIST: } The splitting class for MNIST consists of three thresholds on the average pixel intensity of the 49 $4 \times 4$ blocks of each $28 \times 28$ image. For example, the smallest threshold on the ``first" block for a flattened MNIST image $x$ would return 1 if $(x[0]+x[1]+x[28]+x[29])/4 \leq 42.5$ and 0 otherwise.

\vspace{0.5em}
\noindent \textbf{Avazu CTR: } The splitting class for Avazu CTR were chosen without a fixed number of thresholds for each feature since the range for each feature fluctuates more than that of the other datasets. The splitting class consists of 7 thresholds on $C1$, $100$ thresholds on $C14$, $4$ thresholds on $C15$ and $C16$, $40$ thresholds on $C17$, $10$ thresholds on $C19$ and $C21$, and $15$ thresholds on $C20$.

\subsection{More Experimental Results for Figure \ref{fig:alpha_trainingfraction_test_acc_paper_fig}}
\label{sec:more_results_for_alpha_trainingfraction_fig}

In Figure \ref{fig:alpha_trainingfraction_test_acc_supp}, we now show the effect of LPF, with the same setup as Figure \ref{fig:alpha_trainingfraction_test_acc_paper_fig}, for the rest of the datasets: CreditCard, Skin, MNIST60K, KDDCup, CTR. We see the same overall trend that performance is highest for the non-boundary values of LPF (i.e. when LPF isn't close to 0 or 1).
However, the effect of altering LPF is noticeably different on each combination of dataset, algorithm, $\alpha$ and budgeting function. For instance, on LocalRNM with KDDCup, $\alpha=8.0$ and $\alpha=1.0$ with decay budgeting saw large improvement in performance when LPF was decreased to $0.5$, while the other lines in the plot remained with poor performance for all our values of LPF.

\begin{figure*}
\begin{center}
\centerline{\includegraphics[width=0.8\textwidth]{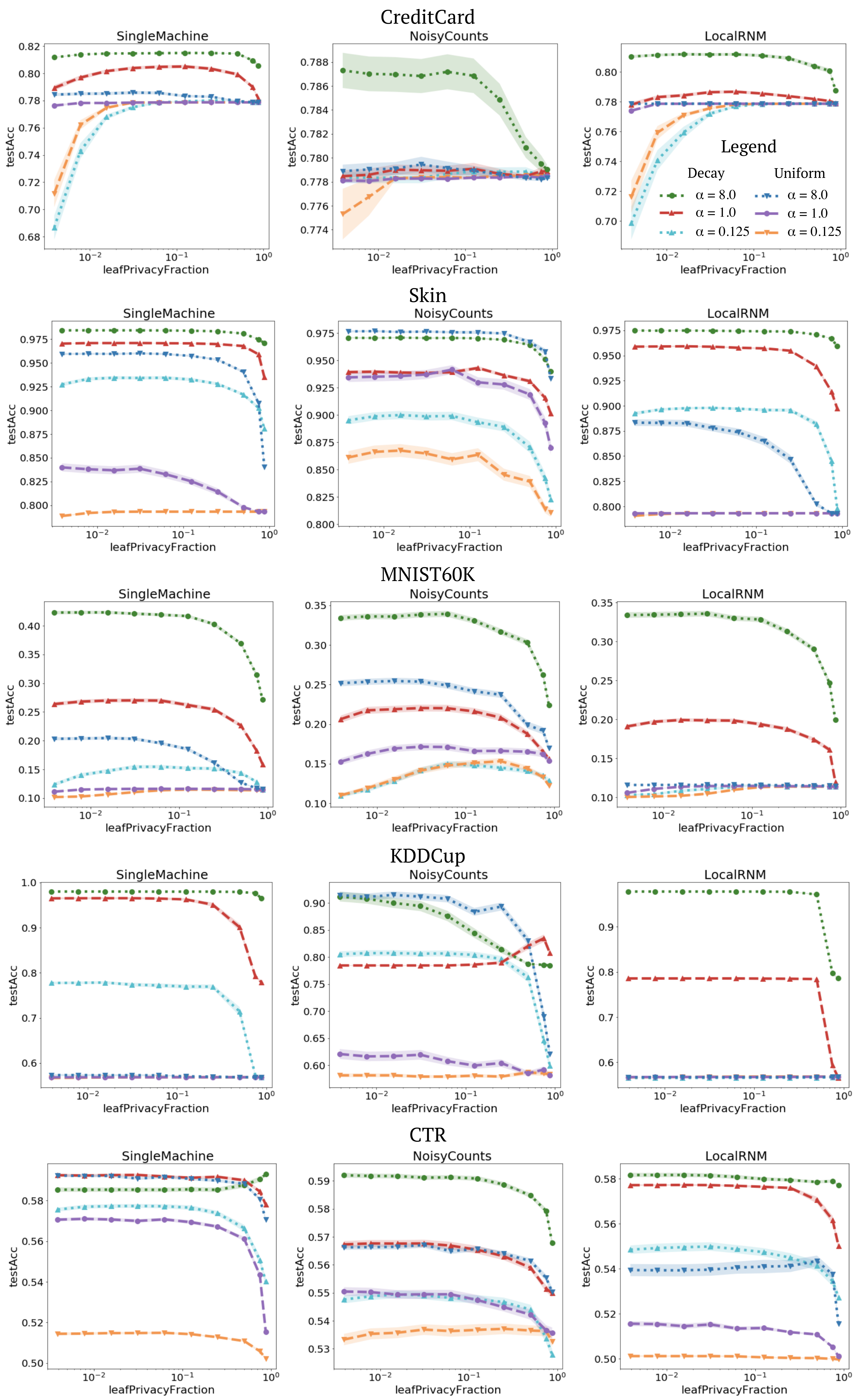}}
\vspace*{-5mm}
\caption{Test accuracy against LPF for decay and uniform budgeting strategies, and varying values of $\alpha \in \{0.125, 1, 8\}$. Same setup as Figure \ref{fig:alpha_trainingfraction_test_acc_paper_fig}, on the rest of the datasets. }
\label{fig:alpha_trainingfraction_test_acc_supp}
\end{center}
\vskip -0.2in
\end{figure*}



\subsection{More Experimental Results for Figure \ref{fig:algos_test_acc_all}}
\label{sec:more_results_for_algos_test_acc_all}

Figure \ref{fig:algo_rest_stats_all} shows the training accuracy privacy curves as well as plots of the depths and sizes of the learned decision trees in Figure \ref{fig:algos_test_acc_all}.
In general, the depth and size of the tree decreases as the privacy budget decreases. This suggests that DP-TopDown is inclined to terminate earlier if it determines that it is no longer productive to make the deeper splits. We also observe an interesting phenomenon that NoisyCounts consistently learns the deepest trees, while LocalRNM consistently learns the shallowest trees.

\begin{figure*}[ht]
\begin{center}
\centerline{\includegraphics[width=0.95\textwidth]{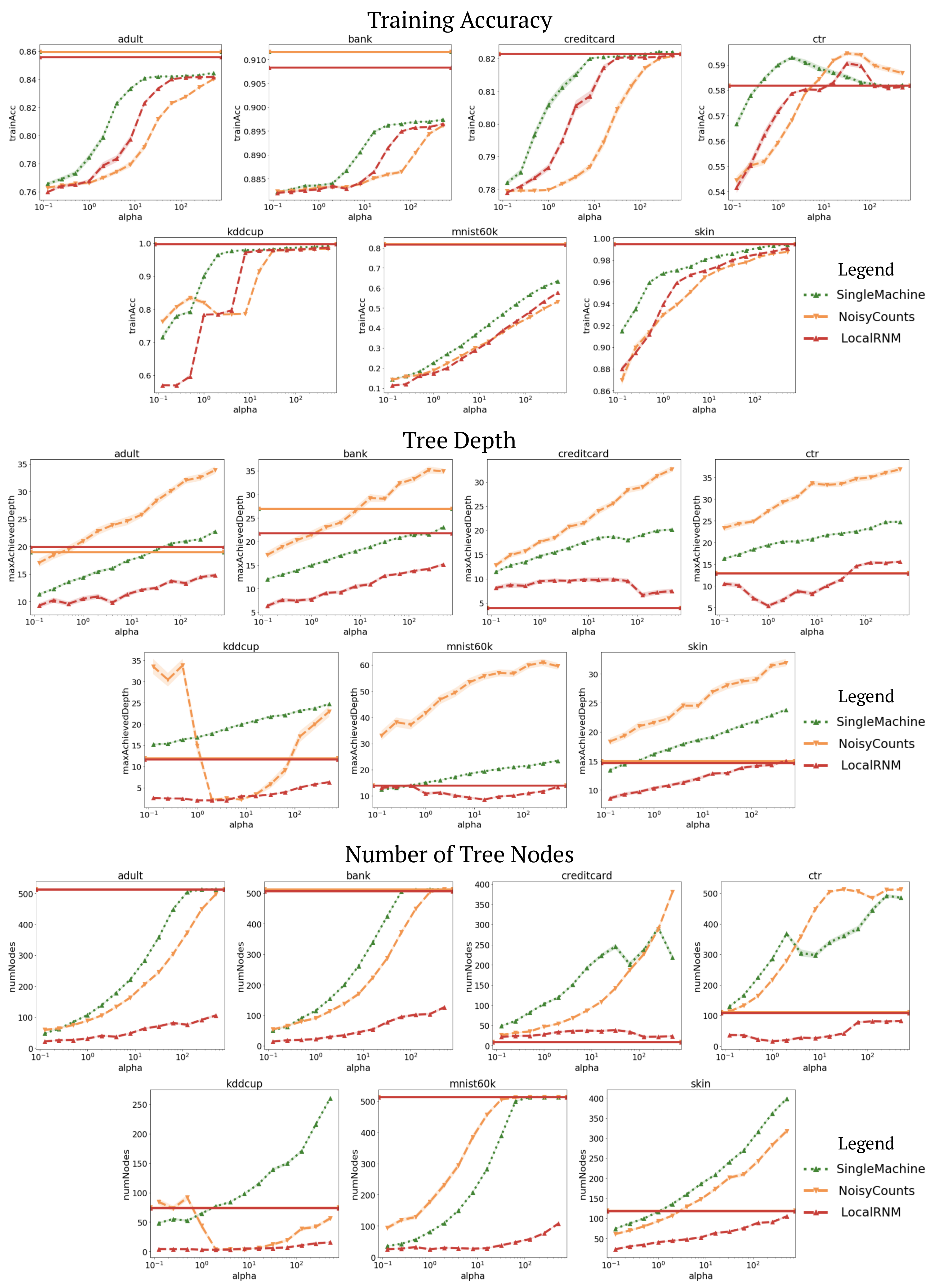}}
\vspace*{-5mm}
\caption{
Training accuracy, tree depth and number of tree nodes under same setup as Figure \ref{fig:algos_test_acc_all}.}
\label{fig:algo_rest_stats_all}
\end{center}
\vskip -0.2in
\end{figure*}

\subsection{More Experimental Results for Figure \ref{fig:datasize_test_acc_all}}
\label{sec:more_results_for_datasize_test_acc_all}
Figure \ref{fig:datasize_train_acc_all} shows the effects of training dataset size on the training accuracy privacy curves. It is essentially the training accuracy analog of Figure \ref{fig:datasize_test_acc_all}.
In these plots, we see the same trend that greater training size leads to better training accuracy, which is what we observed for testing accuracy in Figure \ref{fig:datasize_test_acc_all}. This is especially interesting for training accuracy since smaller datasets are easy to overfit and thus have higher training accuracies, which we indeed observe in the baselines. However, the training privacy curves show that the training accuracy for private decision trees trained on smaller datasets remain lower than that of private decision trees trained on larger datasets. Therefore, the plots in Figure \ref{fig:datasize_train_acc_all} suggest that the effect of dataset size on the performance of private algorithms is strong enough to even overcome the effect of overfitting on small datasets.

\begin{figure*}[ht]
\begin{center}
\centerline{\includegraphics[width=0.95\textwidth]{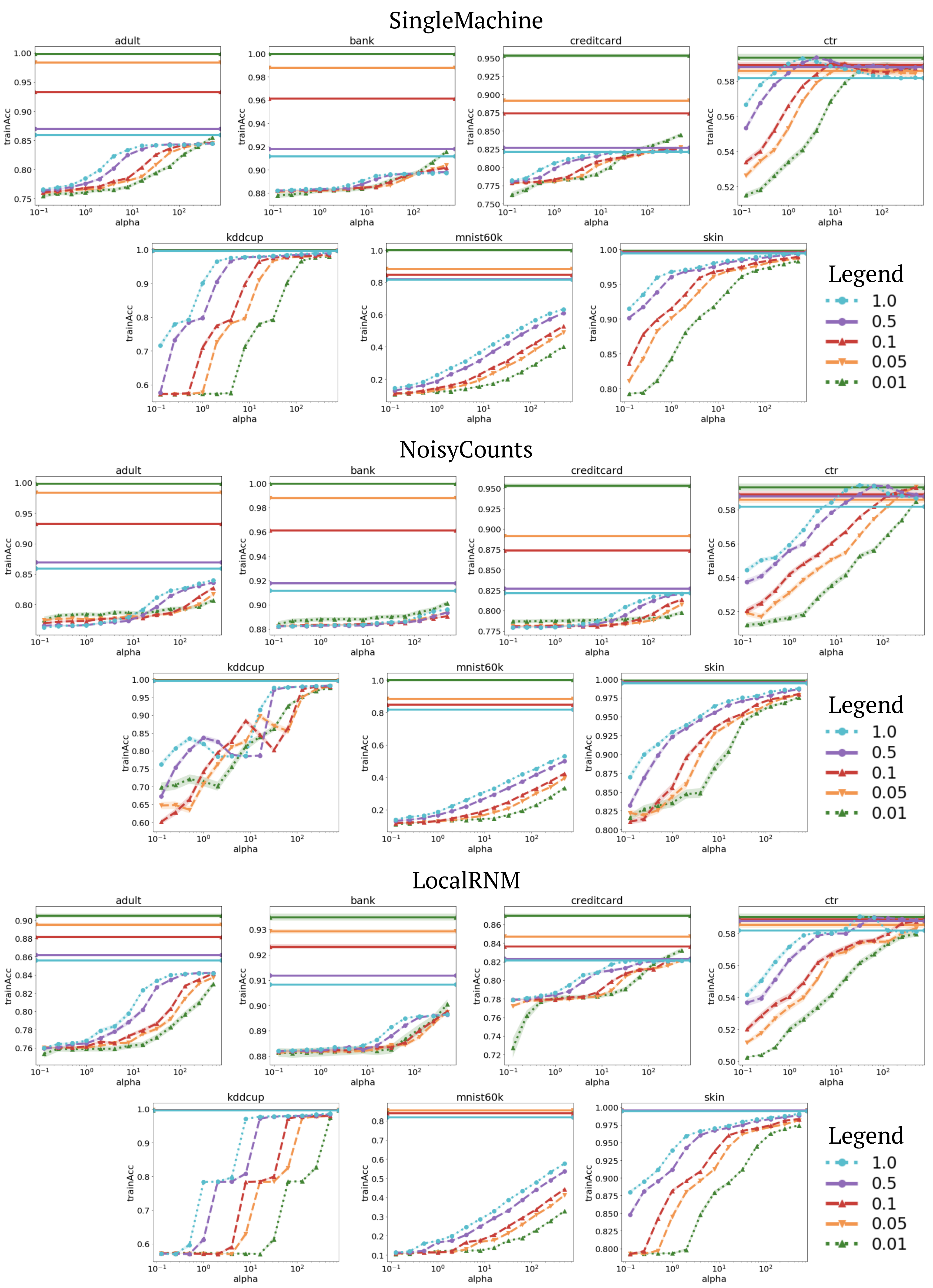}}
\vspace*{-5mm}
\caption{
Training accuracy under same setup as Figure \ref{fig:datasize_test_acc_all}. }
\label{fig:datasize_train_acc_all}
\end{center}
\vskip -0.2in
\end{figure*}

\subsection{Sensitivity of Other Splitting Criteria}
\label{sec:otherCriterion}
We now derive the sensitivity of Gini and Root Gini splitting criteria, which were analyzed in \cite{kearns1999boosting}.
\begin{lemma}
For $a, b, x, y \in [0,1]$, $|ax-by| \leq |a-b|+|x-y|$.
\end{lemma}
We now deduce the global sensitivity of the Gini Criterion $G(q) = 4q(1-q)$.
\begin{lemma}
Suppose $G(q) = 4q(1-q)$, the Gini Criterion. When applied to a dataset of size $m$, then for any $h \in H$, the sensitivity of $J'(S) = J(S,h)$ is $\Delta J' \leq \frac{20}{m}$.
\end{lemma}
\begin{proof}
Use the same notation as in Lemma \ref{lm:entropysensitivity}.
\begin{align*}
    &|J(S,h)-J(S',h)| \\
    &= 4|[r_{y=1}(1-r_{y=1})-r'_{y=1}(1-r'_{y=1})] \\
    &- [r_{h=0}r_{y=1,h=0}(1-r_{y=1,h=0}) \\
    &- r'_{h=0}r'_{y=1,h=0}(1-r'_{y=1,h=0})] \\
    &- [r_{h=1}r_{y=1,h=1}(1-r_{y=1,h=1}) \\
    &- r'_{h=1}r'_{y=1,h=1}(1-r'_{y=1,h=1})]| \\
    &\leq 4[\frac{1}{m} + \frac{2}{m} + \frac{2}{m}] = \frac{20}{m}
\end{align*}
by triangle inequality, the above lemma, and that $|r_*(1-r_*)-r'_*(1-r'_*)| \leq |r_*-r'_*| \cdot |1-(r_*+r'_*)| \leq \frac{1}{m}$.
\end{proof}

\begin{lemma}
For $x,y \in [0,1]$, $|\sqrt{x}-\sqrt{y}| \leq \sqrt{|x-y|}$.
\end{lemma}
\begin{proof}
WLOG $x \geq y$. Then, $(\sqrt{x}-\sqrt{y})^2 = x-2\sqrt{xy}+y \leq x-2y+y = x-y$. Square root both sides.
\end{proof}
We now deduce the global sensitivity of the Root Gini Criterion $G(q) = 2\sqrt{q(1-q)}$.
\begin{lemma}
Suppose $G(q) = 2\sqrt{q(1-q)}$, the Gini Criterion. When applied to a dataset of size $m$, then for any $h \in H$, the sensitivity of $J'(S) = J(S,h)$ is $\Delta J' \leq \frac{10}{m}$.
\end{lemma}
\begin{proof}
Use the same notation as in Lemma \ref{lm:entropysensitivity}.
\begin{align*}
    &|J(S,h)-J(S',h)| \\
    &= 2|[\sqrt{r_{y=1}(1-r_{y=1})}-\sqrt{r'_{y=1}(1-r'_{y=1})}] \\
    &- [r_{h=0}\sqrt{r_{y=1,h=0}(1-r_{y=1,h=0})} \\
    &- r'_{h=0}\sqrt{r'_{y=1,h=0}(1-r'_{y=1,h=0})}] \\
    &- [r_{h=1}\sqrt{r_{y=1,h=1}(1-r_{y=1,h=1})} \\
    &- r'_{h=1}\sqrt{r'_{y=1,h=1}(1-r'_{y=1,h=1})}]| \\
    &\leq 2[\frac{1}{m} + \frac{2}{m} + \frac{2}{m}] = \frac{10}{m}
\end{align*}
by triangle inequality and the above lemma.
\end{proof}

Since the sensitivity is still $\widetilde{O}(\frac{1}{m})$, our sample complexity bounds for entropy also hold for Gini and Root Gini.

\end{document}